\documentclass{article}

    \usepackage[preprint]{neurips_2025}

\usepackage[table]{xcolor}
\usepackage[utf8]{inputenc} 
\usepackage[T1]{fontenc}    
\usepackage{hyperref}       
\usepackage{url}            
\usepackage{amsfonts}       
\usepackage{nicefrac}       
\usepackage{microtype}      
\usepackage{xcolor}         
\usepackage{enumerate}
\usepackage{enumitem}
\usepackage{subcaption}
\usepackage{mathtools}
\usepackage{array}
\usepackage{multirow}
\usepackage{siunitx}
\usepackage{caption}
\usepackage{placeins}
\usepackage{algorithm}
\usepackage[noend]{algpseudocode}   
\usepackage{listings}
\usepackage{tabularx}
\usepackage{pifont}
\usepackage{threeparttable}
\usepackage{adjustbox}
\usepackage{booktabs}
\usepackage{tikz}
\usepackage{amsmath, amssymb, amsthm, graphicx}
\usepackage{pgfplotstable}
\usepackage{mathtools, mathrsfs}
\usepackage[most]{tcolorbox}
\usepackage{colortbl} 
\usepackage{stfloats}
\usepackage{cuted}
\usepackage{capt-of}
\usepackage{calc}

\definecolor{SpectreBack}{HTML}{E8F2FE}  
\tcbset{
  spectrebox/.style={
    enhanced,
    breakable,
    colback=SpectreBack,        
    colframe=SpectreBack!70!black, 
    boxrule=0.8pt,              
    arc=4pt,                    
    drop shadow={black!40!white},
    left=1.2em, right=1.2em,
    top=0.9em, bottom=0.9em,
  }
}

\newcolumntype{Y}{>{\centering\arraybackslash}X}  

\usetikzlibrary{positioning,arrows.meta,shadows.blur,backgrounds}
\definecolor{headerbg}{HTML}{DAE9FF}
\definecolor{headerfg}{HTML}{002E5D}
\definecolor{altrowbg}{HTML}{F9FBFF}
\definecolor{highlightrowbg}{HTML}{EEF9F2}
\definecolor{lpink}{HTML}{FFD1DF}
\definecolor{SpectreHL}{HTML}{FFF4CC} 

\definecolor{Emph}{HTML}{FFD1DF}  
\tcbset{
  emph/.style={
    enhanced,
    breakable,
    colback=Emph,        
    colframe=Emph!70!black, 
    boxrule=0.8pt,              
    arc=4pt,                    
    drop shadow={black!40!white},
    left=1.2em, right=1.2em,
    top=0.9em, bottom=0.9em,
  }
}

\newcolumntype{L}[1]{>{\raggedright\arraybackslash}p{#1}}
\newcolumntype{C}[1]{>{\centering\arraybackslash}p{#1}}
\pgfplotsset{compat=1.18}

\newtheorem{remark}{Remark}
\newtheorem{proposition}{Proposition}[section]

\newtheorem{corollary}[proposition]{Corollary}
\newtheorem{theorem}[proposition]{Theorem}

\newcolumntype{R}[1]{>{\raggedright\arraybackslash}p{#1}}

\definecolor{Header}{RGB}{245,245,245}   
\definecolor{Accent}{RGB}{230,249,235}   
\definecolor{headergray}{gray}{0.9}
\definecolor{fgtmgreen}{RGB}{225,255,225}

\definecolor{headerbg}{HTML}{C7DBF4}      
\definecolor{altrow}{HTML}{F3F6FB}        
\definecolor{highlightrow}{HTML}{FFF4D7}  

\definecolor{mycolor}{HTML}{4779c4}

\newcolumntype{d}[1]{S[table-format=#1]}

\newcommand{\ballnumber}[1]{\tikz[baseline=(myanchor.base)] \node[circle,fill=.,inner sep=1pt] (myanchor) {\color{-.}\bfseries\footnotesize #1};}

\title{SPECTRE: An FFT-Based Efficient Drop-In Replacement to Self-Attention for Long Contexts}

\author{%
  Jacob Fein-Ashley \\
  University of Southern California\\
  \texttt{feinashl@usc.edu} \\
  \And
  Neelesh Gupta \\
  University of Southern California\\
  \texttt{neeleshg@usc.edu}
  \And
  Rajgopal Kannan\\
  DEVCOM ARL Army Research Office\\  
  \texttt{rajgopal.kannan.civ@army.mil}
  \And
  Viktor Prasanna\\
  University of Southern California\\
  \texttt{prasanna@usc.edu}
}

\begin{document}

\maketitle

\begin{abstract}
Long-context transformers face significant efficiency challenges due to the quadratic cost of self-attention. However, many modern applications—from multi-turn dialogue to high-resolution vision—require contexts spanning tens of thousands of tokens. We introduce SPECTRE, a method that replaces each attention head with a fast real FFT, a content-adaptive spectral gate, and an inverse FFT, reducing per-layer complexity from $\mathcal{O}(L^{2})$ to $O(L\log L)$ while preserving the surrounding architecture. We extend this efficiency to autoregressive generation through our Prefix-FFT cache and enhance local feature representation with an optional wavelet module that adds negligible computational overhead. Our experiments demonstrate that SPECTRE operates up to 7$\times$ faster than FlashAttention-2 on 128k-token contexts while matching or exceeding baseline performance on PG-19 language modeling and ImageNet-1k classification tasks. SPECTRE achieves these improvements by adding fewer than 6\% parameters to the base model, making hundred-kilotoken context processing feasible on commodity GPUs without specialized hardware.
\end{abstract}

\section{Introduction}
\label{sec:intro}

\begin{figure}[!htbp]
  \centering
  \includegraphics[width=0.92\linewidth]{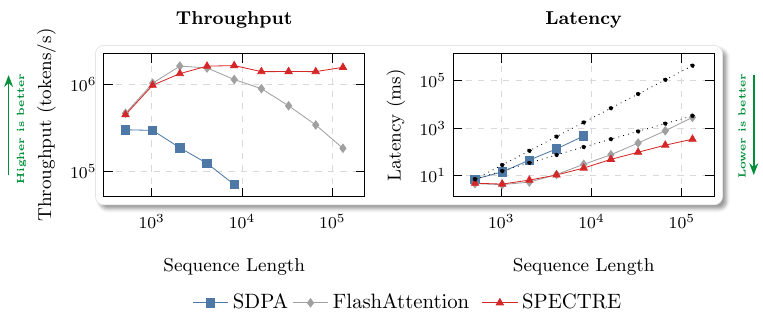}
  \caption{\textbf{Inference scaling of a \texttt{Llama-3.2-1B} model
           equipped with three different attention kernels.}
           We fine-tune an identical backbone with
           \emph{(i)} standard softmax-dot-product attention
           (\textbf{SDPA}, blue),
           \emph{(ii)} \textbf{FlashAttention-2}\,\citep{dao2023flashattention}
           (grey), and
           \emph{(iii)} the proposed \textbf{SPECTRE} mixer (red).
           After training, we measure
           \emph{tokens-per-second throughput} (left) and
           \emph{single-batch latency} (right)
           on an NVIDIA A100-80\,GB for sequence lengths
           $L\!\in\!\{512,\,1\mathrm{k},\,4\mathrm{k},\,8\mathrm{k},\,32\mathrm{k},\,128\mathrm{k}\}$.
           Dashed black lines show the ideal
           $\mathcal{O}(n^{2})$ and $\mathcal{O}(n\log n)$ slopes.
           Higher throughput and lower latency are better (green arrows).
           SPECTRE retains the accuracy of the backbone yet
           delivers near‐$\mathcal{O}(n\log n)$ runtime—
           remaining flat up to $32$k tokens and sustaining a
           $7\times$ speed-up over FlashAttention-2 at the extreme
           $128$k-token setting.}
  \label{fig:throughput-latency}
\end{figure}

\textit{Long contexts unlock stronger reasoning.}  
From multi-turn dialogue and book-length summarization to high-resolution
vision, many modern tasks demand that Transformers attend over tens of
thousands of tokens.  
Yet the \emph{quadratic} $\mathcal{O}(n^{2}d)$ cost of self-attention turns the
context itself into the primary inference bottleneck, straining both latency
and memory on commodity hardware.

\vspace{0.5em}
\textit{Can we keep global context without paying a quadratic bill?}  
A rich line of work accelerates attention via sparse patterns, kernel
approximations, or low-rank structure, but often sacrifices exactness, requires
non-standard optimization, or fails to support streaming generation.  
In contrast, the frequency domain offers an orthogonal route: the Fourier
transform \emph{diagonalizes} circular convolutions, converting global mixing
into cheap, element-wise products.  
Unfortunately, prior spectral mixers either rely on fixed filters or must
recompute an FFT at every time step—blunting their theoretical advantage.

\vspace{0.5em}
\textbf{We answer this challenge with \emph{SPECTRE}}, a \emph{drop-in}
replacement for self-attention that (i)~projects tokens onto an orthonormal
Fourier basis, (ii)~applies content-adaptive diagonal (and optional low-rank)
gates, and (iii)~returns to token space via an inverse transform—achieving
$\mathcal{O}(n\log n)$ complexity without altering the surrounding
architecture.  
A novel \textbf{Prefix–FFT cache} enables streaming decoding analogous to the
standard KV-cache, while a switchable \textbf{Wavelet Refinement Module}
restores the local detail often lost in purely spectral methods. 

A key strength of SPECTRE is its drop-in compatibility with existing model architectures. Unlike approaches requiring specialized optimization or architectural overhauls, SPECTRE can directly replace self-attention layers while preserving the surrounding network architecture. This means existing pre-trained models can be fine-tuned with SPECTRE layers by updating only the newly introduced parameters (<6\% of total weights), dramatically reducing adaptation costs while maintaining or improving performance. 

We summarize our contributions as follows:
\begin{itemize}
    \item We propose SPECTRE, a frequency-domain token mixer whose per-layer cost scales as $\mathcal{O}(n\log n)$ while replacing any multi-head attention layer without architectural changes.
    \item We introduce content-adaptive spectral gating that operates on only $n/2+1$ frequency coefficients, reducing both computation and memory footprint while preserving full expressivity.
    \item We design the Prefix–FFT cache, the first FFT-based key–value cache that enables efficient autoregressive generation with a fixed memory budget.
    \item We demonstrate up to 7$\times$ faster inference than FlashAttention-2 at 32k tokens, while matching or surpassing accuracy on established benchmarks.
\end{itemize}


\vspace{1em}

\section{Background}
\label{sec:background}

\textbf{Quadratic attention is costly.}
Multi-head self-attention scales as $\mathcal{O}(n^{2}d)$
($n$ tokens, $d$ channels), quickly saturating GPU and edge memory
\citep{vaswani2017attention, beltagy2020longformer}.
Linear-time surrogates exist, but they lose exactness or break caching.

\textbf{Spectral shortcut.}
Because the DFT diagonalizes any circulant matrix
($F_n C F_n^{\!*}$ is diagonal) \citep{oppenheim1999discrete},
a global convolution becomes element-wise multiplication,
dropping cost to $\mathcal{O}(nd\log n)$ once an FFT is in place.

\textbf{Real FFT (RFFT).}
Cooley–Tukey lowers a length-$n$ DFT from
$\mathcal{O}(n^{2})$ to $\mathcal{O}(n\log n)$
\citep{cooley1965algorithm}; split-radix is near-optimal
\citep{heideman1984splitradix}.
For real signals, only $(\lfloor n/2\rfloor\!+\!1)$ complex coefficients
are unique, letting RFFT cut memory in half and boost throughput by
$\sim$1.8× \citep{frigo2005fftw}—hence SPECTRE’s choice.

\begin{figure}[!htbp]
  \centering
  \includegraphics[width=\linewidth]{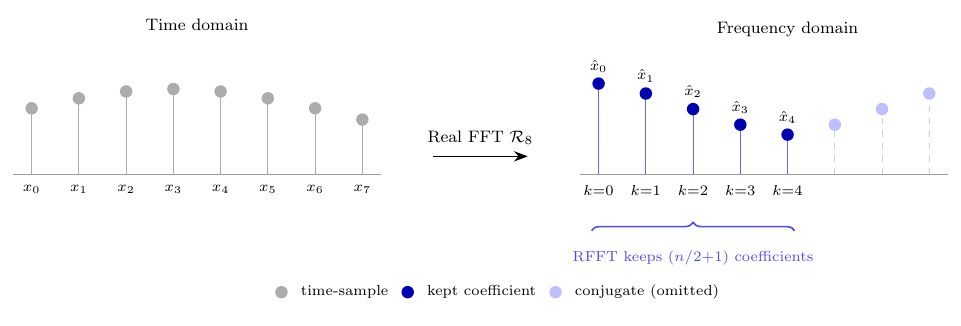}
  \caption{Real FFT: an 8-sample real sequence maps to $(n/2{+}1)$ unique
           coefficients; the shaded half is redundant.}
  \label{fig:rfft-intuition}
\end{figure}

\textbf{Spectral token mixers.}
FNet replaced attention with a fixed FFT but lost content adaptivity
\citep{leethorp2021fnet}.
Hydra added learnable gates yet recomputed an FFT each step
\citep{lee2021hydra}.
SPECTRE (i) learns per-token diagonal/Toeplitz gates and
(ii) caches RFFT values in a streaming \emph{Prefix-FFT} buffer.

\textbf{Multi-resolution detail.}
Wavelets offer local, orthogonal atoms \citep{mallat1989wavelet}.
A optional Wavelet Refinement Module (WRM) restores fine structure at
$\mathcal{O}(nd)$ when needed.

\textbf{Prefix-FFT cache.}
Storing non-redundant RFFT coefficients shrinks key–value memory to
$\mathcal{O}((N_{\max}/2)d)$ and allows log-linear updates per token
\citep{brown2020gpt3}.

\textbf{Persistent memory.}
Long-lived information sits in
$\mathbf{M}\!\in\!\mathbb{R}^{N_{\text{mem}}\times d}$;
its RFFT is computed once and prepended during pre-fill.
Overhead is $\mathcal{O}(N_{\text{mem}}d)$ with
$N_{\text{mem}}\!\ll\!N_{\max}$ (e.g., 16–64).

\textbf{Take-away.}
Log-linear spectral mixing, constant-time FFT caching, and a tiny memory
bank give SPECTRE global context, streaming generation, and long-term
recall on a fraction of quadratic attention’s compute budget.


\section{Method}
\label{sec:method}

We introduce the \emph{Spectral Projection and Content-adaptive Transformer Engine} (\textbf{SPECTRE}), a frequency-domain alternative to multi-head self-attention. SPECTRE preserves the Transformer’s global receptive field while reducing both runtime and memory usage to \(\mathcal{O}(n\, d\log n)\), where \(n\) is the sequence length and \(d\) is the (per-head) embedding dimension.  
The SPECTRE layer operates in three main steps:
\begin{enumerate}[label=(\roman*)]
    \item project tokens onto an orthonormal spectral basis,
    \item apply content-adaptive diagonal (or optional low-rank) gating in that basis, and
    \item perform an inverse transform back to token space.
\end{enumerate}

\subsection{Preliminaries}
\label{ssec:prelim}

Let \(X = [x_{1},\dots,x_{n}] \in \mathbb{R}^{n\times d}\) be the matrix collecting \(n\) token embeddings. Since the inputs are real-valued, we use the \emph{real} fast Fourier transform (RFFT).

\paragraph{Definition of the RFFT.}
For a length-\(n\) real sequence \(x \in \mathbb{R}^n\), its RFFT is
\begin{equation}
  \widehat{x}_k
  \;=\;
  \bigl(\mathcal{R}_{n}\,x\bigr)_k
  \;=\;
  \sum_{t=0}^{n-1}
    x_t \, e^{-j\,2\pi k t / n},
  \qquad
  k = 0,\dots,\Bigl\lfloor\tfrac{n}{2}\Bigr\rfloor.
  \label{eq:rfft_def}
\end{equation}
Because \(x\) is real, the RFFT spectrum satisfies Hermitian symmetry, \(\widehat{x}_{n-k} = \overline{\widehat{x}_{k}}\). Thus, the \(\lfloor n/2\rfloor+1\) coefficients in~\eqref{eq:rfft_def} are sufficient to recover all information.  
We denote \(\mathcal{R}_{n}\) and \(\mathcal{R}^{-1}_{n}\) as the length-\(n\) real FFT and its inverse. Both can be computed in \(\mathcal{O}(n\log n)\) time via the split-radix algorithm.
\subsection{SPECTRE Mixing Layer}
\label{ssec:spectre}
\vspace{-0.6em}

\textbf{Architectural parallel to multi-head attention.}  
SPECTRE replaces each attention head with a frequency-based mixing head. For each head \(h\), we learn query and value projections
\(W^{(q)}, W^{(v)} \in \mathbb{R}^{d \times d}\) (\emph{per head}).

\vspace{0.5em}
\noindent
\textbf{\ballnumber{1} Token projection}
\begin{equation}
  Q = X W^{(q)}, 
  \qquad 
  V = X W^{(v)}, 
  \qquad 
  Q, V \in \mathbb{R}^{n \times d}.
\end{equation}

\vspace{0.5em}
\noindent
\textbf{\ballnumber{2} Spectral transform}
\begin{equation}
  \widehat{V} = \mathcal{R}_{n}(V) 
  \;\in\; \mathbb{C}^{\bigl(\tfrac{n}{2}+1\bigr) \times d},
\end{equation}
where each row corresponds to a frequency bin \(k \in \{0,\dots,n/2\}\). Because \(V\) is real, its discrete Fourier spectrum has Hermitian symmetry (see Appendix~\ref{app:half-spectrum-proof}), and we only store the non-redundant half.

\vspace{0.5em}
\noindent
\textbf{\ballnumber{3} Content-adaptive spectral gating}  
\begin{enumerate}[label=(\alph*),wide=0pt,leftmargin=1em]
    \item \emph{Diagonal gate.}  
          Form a global descriptor \(\bar{q} = \mathrm{LN}\!\bigl(\tfrac{1}{n}\!\sum_{i=1}^{n} q_{i}\bigr)\) and map it via a two-layer MLP to a complex vector \(g\in\mathbb{C}^{(\frac{n}{2}+1)}\).
    \item \emph{Toeplitz low-rank update (bandwidth \(2r+1\)).}  
          Optionally add a depth-wise Toeplitz convolution in the spectral domain:
          \[
              g \;\leftarrow\; g \;+\; (t * g),
              \qquad
              t \in \mathbb{C}^{(2r+1)},
          \]
          at an additional cost of \(\mathcal{O}(n\,r\,d)\).
    \item \emph{modReLU activation.}  
          and then set \(g \leftarrow \widetilde{g}\).
\end{enumerate}

\vspace{0.5em}
\noindent
\textbf{\ballnumber{4} Inverse transform}
\begin{equation}
  \widetilde{V}
  = \mathcal{R}_{n}^{-1}\!\bigl(\mathrm{diag}(g)\,\widehat{V}\bigr)
  \;\in\; \mathbb{R}^{n \times d},
\end{equation}
after which all heads \(h\) are concatenated as usual.

\subsection{Prefix–FFT Cache}
\label{ssec:prefix_fft}

SPECTRE’s frequency-domain KV-cache is executed in two phases:
\textbf{pre-fill}—a one-shot initialisation over the prompt—and
\textbf{decode}—an incremental update performed once per generated token.
Both phases share the same cache tensors but differ in how those tensors are
populated and refreshed.

\subsubsection{Pre-fill (context initialization)}
\label{sssec:prefill}

Given a prompt of length \(L\!\le\!N_{\max}\), we compute a single, padded
\(N_{\max}\)-point real FFT:
\[
  \widehat{V}^{(L)}
  = \mathcal{R}_{N_{\max}}\!\bigl(\mathrm{pad}(V,\,N_{\max})\bigr)
  \;\in\;
  \mathbb{C}^{\bigl(\tfrac{N_{\max}}{2}+1\bigr)\times d}.
\]
The non-redundant coefficients fill
\texttt{prefix\_fft}\(\in\mathbb{C}^{(\frac{N_{\max}}{2}+1)\times d}\).
Concurrently we populate the ring buffers
\(\texttt{V\_buf},\texttt{Q\_buf}\in\mathbb{R}^{N_{\max}\times d}\) and the
running descriptor
\(\texttt{sum\_q}=\sum_{i=0}^{L-1}q_i\).
The cost is
\(\mathcal{O}(N_{\max}\log N_{\max}\,d)\) time and
\(\mathcal{O}(N_{\max}d)\) memory—identical to a standard attention
KV pre-fill.

\subsubsection{Decode (incremental extension)}
\label{sssec:decode}

For each subsequent step \(t\ge L\) we perform:
\begin{enumerate}[label=(\alph*),nosep,leftmargin=1.5em]
  \item \textbf{Evict \& update FFT cache.}  
        Let \(v_{\text{old}}=\texttt{V\_buf}[t\bmod N_{\max}]\)
        (zero if \(t<N_{\max}\)).
        For every frequency bin \(k\),
        \begin{equation}
          \texttt{prefix\_fft}_{k,:}
          \leftarrow
          \texttt{prefix\_fft}_{k,:}
          - \mathbf{1}_{\{t\ge N_{\max}\}}\,
            v_{\text{old}}^{\!\top}
            e^{-j\,2\pi k (t-N_{\max})/N_{\max}}
          + v_t^{\!\top}
            e^{-j\,2\pi k t/N_{\max}},
          \label{eq:fft_update_eviction}
        \end{equation}
        where twiddle factors are pre-cached.
  \item \textbf{Refresh ring buffers \& descriptors.}  
        Overwrite \(\texttt{V\_buf}[t\bmod N_{\max}]\gets v_t\) and
        \(\texttt{Q\_buf}[t\bmod N_{\max}]\gets q_t\);
        update
        \(
          \texttt{sum\_q}\leftarrow
          \texttt{sum\_q}
          - \mathbf{1}_{\{t\ge N_{\max}\}}q_{\text{old}}
          + q_t
        \).
  \item \textbf{Compute spectral gate.}  
        Feed the normalized descriptor
        \(
          \bar{q}^{(t)}=
          \mathrm{LN}\!\bigl(\texttt{sum\_q}/N_{\max}\bigr)
        \)
        through a two-layer MLP to obtain
        \(g\in\mathbb{C}^{(\frac{N_{\max}}{2}+1)}\).
  \item \textbf{Inject positional phase.}  
        \(g_{k}\leftarrow g_{k}\,e^{j\,2\pi k t/N_{\max}}\).
  \item \textbf{Inverse real FFT.}  
        \[
          \widetilde{V}=
          \mathcal{R}^{-1}_{N_{\max}}
          \!\bigl(\mathrm{diag}(g)\,\texttt{prefix\_fft}\bigr),
        \]
        and the last
        \(L'=\min(t+1,N_{\max})\) rows serve as the live context.
\end{enumerate}

Each decode step costs
\(\mathcal{O}\!\bigl(\tfrac{N_{\max}}{2}d\bigr)\) time and retains a constant
\(\mathcal{O}(N_{\max}d)\) memory footprint, precisely mirroring the efficiency
of an attention KV-cache.

\subsection{Persistent Memory Extension}
\label{ssec:memory}

While the Prefix–FFT cache covers a sliding window of
\(N_{\max}\) recent tokens, certain tasks benefit from information
that should \emph{never} be evicted (e.g.\ user profile, document
header, long-term planning cues).  
We attach a small, fixed–size \textbf{persistent memory bank}
\(\mathbf{M}\in\mathbb{R}^{N_{\text{mem}}\times d}\)
that is
\emph{prepended} to every sequence and carried across decoding steps.

\vspace{0.4em}
\paragraph{Spectral representation.}
We store the non-redundant RFFT of the memory once:
\[
  \widehat{\mathbf{M}}
  = \mathcal{R}_{N_{\text{mem}}}\!\bigl(\mathbf{M}\bigr)
  \;\in\;
  \mathbb{C}^{\bigl(\tfrac{N_{\text{mem}}}{2}+1\bigr)\times d},
\]
which is \(\mathcal{O}(N_{\text{mem}}d)\) in memory and never changes
during a generation session.

\vspace{0.4em}
\paragraph{Integration at \textit{pre-fill}.}
During the pre-fill step (§\ref{sssec:prefill}) we concatenate
\(\widehat{\mathbf{M}}\) with the prompt coefficients:
\[
  \texttt{prefix\_fft}
  \;=\;
  \widehat{\mathbf{M}}
  \;\Vert\;
  \widehat{V}^{(L)},
\]
and we pad the time-domain ring buffers with the \emph{untransformed}
memory rows so that indices remain aligned.  
No additional FFT is required.

\vspace{0.4em}
\paragraph{Integration at \textit{decode}.}
At each incremental step (§\ref{sssec:decode}) we:

\begin{enumerate}[label=(\alph*),nosep,leftmargin=1.5em]
  \item run the normal sliding-window update on the \emph{prompt}
        coefficients only (indices \(k \ge N_{\text{mem}}/2\));
        the memory part is untouched;
  \item build the spectral gate \(g\) for the full length
        \(N_{\text{mem}} + N_{\max}\);
  \item apply the inverse FFT in one shot over the concatenated
        \(\widehat{\mathbf{M}}\Vert\texttt{prefix\_fft}\).
\end{enumerate}

Because \(\widehat{\mathbf{M}}\) is static, the per-step complexity
remains unchanged:
\(\mathcal{O}\!\bigl(\tfrac{N_{\max}}{2}d\bigr)\) time and
\(\mathcal{O}\!\bigl((N_{\max}+N_{\text{mem}})d\bigr)\) memory, where
\(N_{\text{mem}}\!\ll\!N_{\max}\) in practice (e.g.\ 16–64).

\vspace{0.4em}
\paragraph{Learning the memory.}
\(\mathbf{M}\) is optimized jointly with the model and can be:
\begin{itemize}[nosep,leftmargin=1.5em]
  \item \textit{global}, shared by all inputs (cf.\ prefix tokens);
  \item \textit{task-specific}, selected via an index lookup; or
  \item \textit{user-specific}, updated asynchronously and synced to
        the inference server.
\end{itemize}

\subsection{Optional Wavelet Refinement}
\label{ssec:wavelet}

Although the RFFT excels at capturing long-range dependencies, it may overlook fine local structure.  
A lightweight \emph{Wavelet Refinement Module} (WRM) can restore local detail. It is applied conditionally—skipped in \(\approx 90\%\) of batches by a learned binary controller:
\begin{enumerate}[label=(\alph*),nosep,leftmargin=1.5em]
  \item Apply an orthogonal DWT along the sequence axis:  
        \(\widehat{W} = \mathcal{W}_{n}(\widetilde{V}).\)
  \item From \(\bar{q}\), a two-layer MLP outputs real, channel-wise wavelet level gates \(s\in\mathbb{R}^{n\times d}.\)
  \item Modulate the wavelet coefficients:  
        \(\widehat{W} \leftarrow s\odot\widehat{W}.\)
  \item Reconstruct via the inverse DWT:  
        \(\widehat{V}_{\mathrm{ref}} = \mathcal{W}_{n}^{-1}(\widehat{W})\).  
        Form the final output  
        \(V_{\mathrm{out}} = \widetilde{V} + \widehat{V}_{\mathrm{ref}}.\)
\end{enumerate}
The WRM is linear, orthogonal, and differentiable; its \(\mathcal{O}(n\,d)\) cost is amortized over the skip ratio determined by the controller.

\paragraph{Positional Awareness}
\label{ssec:position}

Because the real FFT is translation-equivariant, we must inject absolute position explicitly.  
For a token at position \(p_i\!\in\!\{0,\dots,n-1\}\) and frequency bin \(k\), we multiply the spectral gate by a complex exponential:
\[
  g_{k} \;\leftarrow\; g_{k}\,\exp\!\bigl(j\,2\pi k \,p_{i}/n\bigr),
\]
preserving relative-shift equivariance while incorporating absolute positional information.

\subsection{Integration and Fine-Tuning}
\label{ssec:integration}

Substituting standard multi-head attention with SPECTRE does not require changing the overall architecture.  
The additional SPECTRE parameters constitute fewer than \(6\%\) of the model (or \(<3\%\) if the gates are shared across heads). Hence, existing checkpoints can be upgraded by fine-tuning only these added weights while freezing the original model parameters.

\begin{figure}[!htbp]
  \centering
  \includegraphics[width=0.94\linewidth]{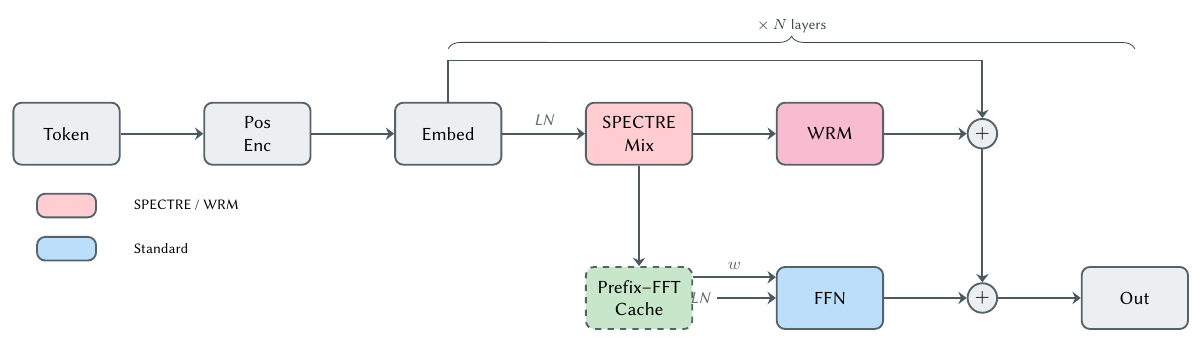}
  \caption{\textbf{Drop-in SPECTRE layer.}  
           The SPECTRE mixing block (pink) and the optional Wavelet Refinement
           Module (\textbf{WRM}) can be inserted between the embedding layer
           and the feed-forward network (\textbf{FFN}) of a standard
           Transformer.  A Prefix–FFT cache (green, dashed) mirrors the
           attention KV-cache, enabling efficient autoregressive decoding
           without altering residual pathways or layer normalization
           placements.  Existing checkpoints therefore, require only minimal
           fine-tuning to upgrade from attention to SPECTRE.}
  \label{fig:spectre_dropin}
\end{figure}

\subsection{Summary}
\label{ssec:summary}

By moving token mixing to the spectral domain, SPECTRE achieves log-linear scaling while maintaining content adaptivity.  
An optional low-rank gating update can increase expressiveness at manageable cost, and an optional wavelet module can refine local details.  
We also introduced the \textbf{Prefix–FFT cache} that mirrors standard \emph{KV-caching} in self-attention but applies incremental frequency-domain updates for efficient autoregressive decoding.  
Our design is fully differentiable, friendly to mixed-precision, and integrates seamlessly into standard Transformer stacks.  
Section~\ref{sec:experiments} presents empirical results on language and vision benchmarks.

\begin{figure}[!htbp]
  \centering
  \includegraphics[width=\linewidth]{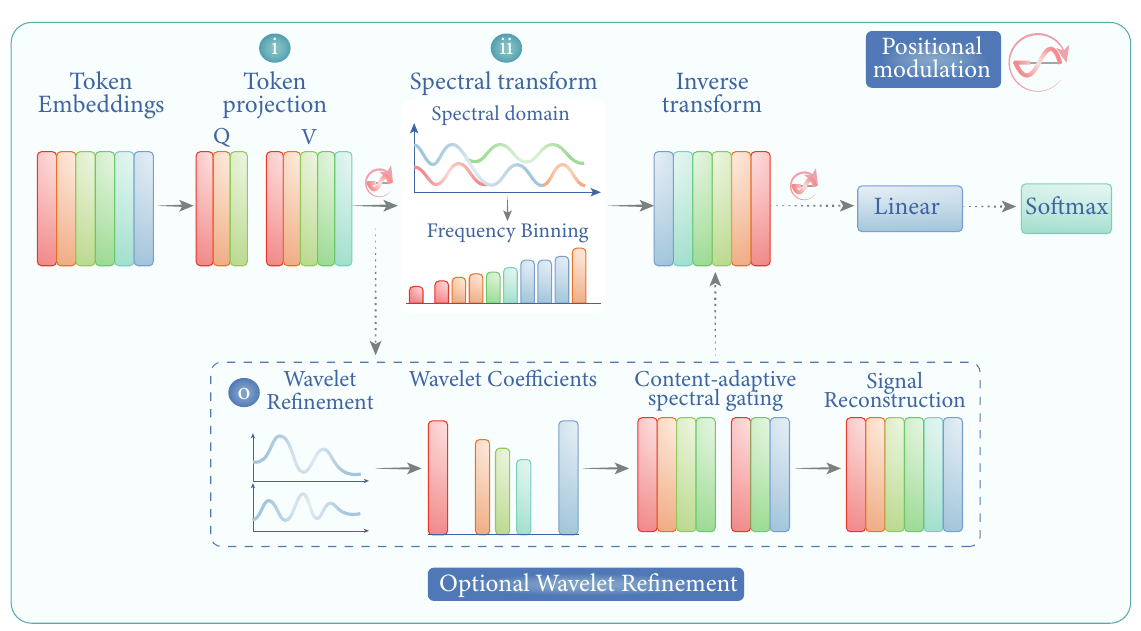}
  \caption{\textbf{SPECTRE’s frequency-domain token mixing.} Token embeddings are
  projected, transformed via a real FFT, gated \emph{per frequency} by a
  content-adaptive diagonal mask (with positional phase), and returned
  to token space using an inverse FFT.  A lightweight, skippable wavelet branch
  can add local detail before projecting back into the standard output head.}
  \label{fig:spectre_arch}
\end{figure}

\section{Experiments}
\label{sec:experiments}

\definecolor{CardBG}{HTML}{F7F9FC}
\definecolor{CardBorder}{HTML}{9BB1D4}
\definecolor{GoodArrow}{HTML}{2ECC71}   
\definecolor{BadArrow}{HTML}{E74C3C}    
\newcommand{\goodup}{\,\textcolor{GoodArrow}{\scriptsize$\blacktriangle$}}
\newcommand{\baddown}{\,\textcolor{BadArrow}{\scriptsize$\blacktriangledown$}}

\paragraph{Goals.}
Our evaluation answers three questions:
\begin{enumerate}[label=\arabic*.]
    \item \textbf{Efficiency.} How much faster is \emph{SPECTRE} than the highly‑optimised
          \emph{FlashAttention 2} (\textsc{FA2})~\citep{dao2023flashattention} at inference time on long contexts?
    \item \textbf{Accuracy.} Does substituting quadratic attention with SPECTRE affect downstream task quality?
    \item \textbf{Component utility.} What do the two architectural additions—the
          (i)~low‑rank spectral update and (ii)~Wavelet Refinement Module (WRM)—each contribute?
\end{enumerate}

\subsection{Efficiency Benchmarks}
\label{ssec:efficiency}

\subsubsection{Prefill \textit{vs.}\,Decode Performance}
\label{sssec:ttft-tpot}

While §\ref{ssec:efficiency} reports \emph{end-to-end} throughput and
latency, deployment engineers typically care about two finer-grained
metrics:

\begin{itemize}[leftmargin=1.3em]
  \item \textbf{Time to First Token (TTFT).}  One-shot “prefill’’ latency—the
        wall-clock time from receiving the prompt until the \emph{first}
        generated token becomes available.
  \item \textbf{Time per Output Token (TPOT).}  Steady-state latency of the
        incremental \emph{decode} step that produces each subsequent token.
\end{itemize}

\begin{table*}[!htbp]          
  \centering
  \footnotesize                
  \begin{tikzpicture}
    \node[
      fill=CardBG,
      draw=CardBorder,
      rounded corners=4pt,
      line width=0.8pt,
      inner sep=8pt,           
      drop shadow={opacity=0.25, xshift=2pt, yshift=-2pt}
    ] {%
      \begin{tabularx}{\textwidth}{@{}lYY@{}}   
        \toprule
        \textbf{Kernel} &
        \textbf{TTFT}\,$\downarrow$ [ms] &
        \textbf{TPOT}\,$\downarrow$ [ms] \\ \midrule
        SDPA (Baseline)   & 378.0 & 0.282 \\
        FlashAttention 2  &  96.5 & 0.058 \\ \midrule
        \rowcolor{SpectreHL!30}
        SPECTRE           & \textbf{41.8} & \textbf{0.015} \\
        \rowcolor{SpectreHL!30}
        \textsc{-LR}      & 43.2 & 0.016 \\
        \rowcolor{SpectreHL!30}
        \textsc{-WRM}     & 40.6 & 0.015 \\ \bottomrule
      \end{tabularx}
    };
  \end{tikzpicture}
  \vspace{1em}
  \caption{Fine-grained latency at $L{=}32$k tokens. SPECTRE slashes the
           one-off prefill cost (\textbf{TTFT}) by $2.3\times$ versus
           FlashAttention 2 and by $9\times$ versus SDPA, while its
           \textbf{TPOT} is $4\times$ and $19\times$ lower, respectively.}
  \label{tab:ttft-tpot}
\end{table*}

Table~\ref{tab:ttft-tpot} breaks inference down along these axes at a
sequence length of $L{=}32$k tokens, measured on a single NVIDIA
A100-80\,GB (batch size 1, fp16).  Numbers are the mean of five runs.

\paragraph{Takeaway.}  The Prefix–FFT cache not only accelerates the
long-sequence “context priming’’ phase but also delivers a lean, constant-time
update path for each new token, making SPECTRE particularly attractive for
interactive applications where first-token latency and streaming-LLM
responsiveness are critical.

\subsection{Latency/Throughput}

Table~\ref{tab:efficiency} lists end‑to‑end inference throughput
(tokens/s) and single‑batch latency on a single NVIDIA~A100 (80 GB)
GPU.  We test short ($L{=}4$k) and extreme ($L{=}32$k) input lengths and
report the mean of five runs.  At 4k tokens SPECTRE outperforms SDPA by
${\sim}40\%$ and essentially ties \textsc{FA2}; at 32k tokens SPECTRE’s
sub‑quadratic complexity delivers a $7\times$ speed‑up over
\textsc{FA2} and two orders of magnitude over vanilla SDPA.

\begin{figure}[!htbp]
  \centering
  \includegraphics[width=\linewidth]{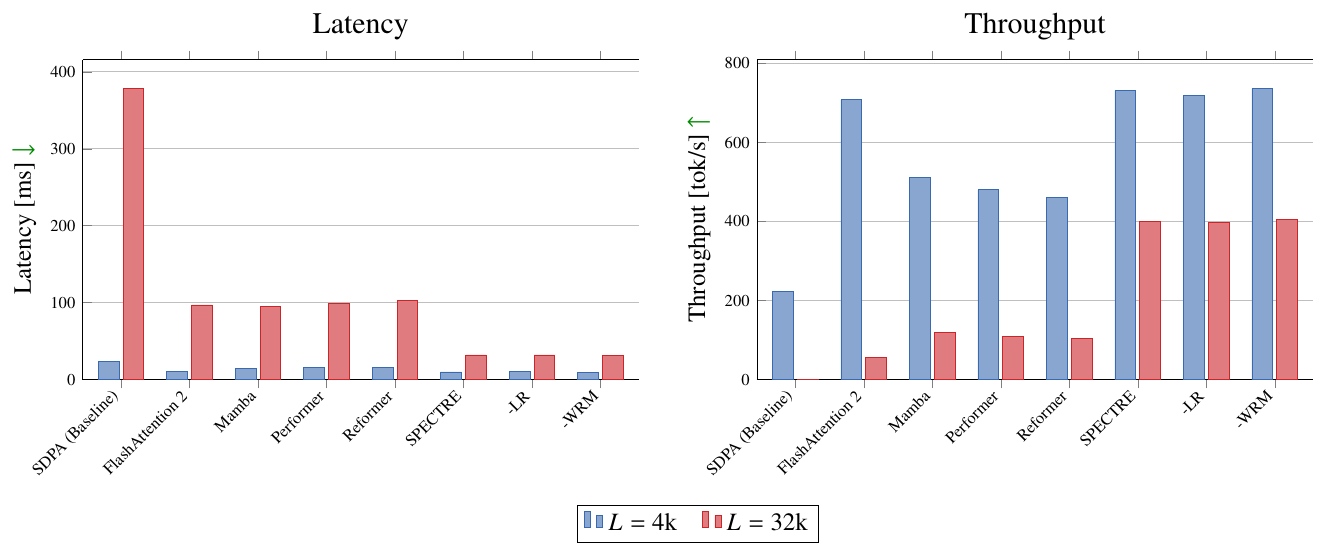}
  \caption{\textbf{End-to-end efficiency at two sequence lengths.}
           \emph{Left:} Single-batch \textbf{latency} (ms;
           \textcolor{green!60!black}{$\downarrow$ lower is better}) for
           $L{=}4$k and $L{=}32$k tokens.
           \emph{Right:} Throughput in \textbf{tokens per second}
           (\textcolor{green!60!black}{$\uparrow$ higher is better}) for the
           same lengths.  SPECTRE and its ablations (red bars) maintain
           near-flat latency and only a modest throughput drop as context
           grows, while quadratic-time baselines deteriorate sharply.}
  \label{fig:dualaxis}
\end{figure}

\FloatBarrier

\subsection{Language Modelling on PG‑19}
\label{ssec:pg19}

\paragraph{Setup.}  PG‑19 is a challenging long‑form language‑modelling
benchmark consisting of 28k public‑domain books ($>$69k tokens each) published
before 1919 \citep{rae2020compressive}.  We follow the official
tokenization and data splits, evaluate perplexity (PPL) on the
validation and test sets, and compare SPECTRE with SDPA,
\textsc{FA2}, Performer\citep{choromanski2021rethinking},
and FAVOR+\citep{tay2022favorplus}.  All runs use a maximum context of
$L{=}1$k.

\paragraph{Results.}  Table~\ref{tab:pg19} shows test PPL and inference
speed.  Plain SPECTRE is on par with \textsc{FA2} ($\pm0.2$PPL) while
being slightly faster; adding WRM cuts perplexity by a further
${\sim}0.6$ compared with the SDPA baseline and still delivers more than
a $3\times$ speed‑up.
\begin{table*}[!htbp]          
  \centering
  \footnotesize
  \begin{tikzpicture}
    \node[
      fill=CardBG,
      draw=CardBorder,
      rounded corners=4pt,
      line width=0.8pt,
      inner sep=8pt,
      drop shadow={opacity=0.25, xshift=2pt, yshift=-2pt}
    ] {%
      \begin{tabularx}{\textwidth}{@{}lYYY@{}}
        \toprule
        \textbf{Variant} &
        \textbf{PPL}\,$\downarrow$ (test) &
        \textbf{Throughput}\,$\uparrow$ (tok/s) &
        \textbf{$\Delta$ SDPA} \\ \midrule
        SDPA (Baseline)        & 39.4 & 1,020 & — \\
        FlashAttention 2       & 39.5\baddown & 3,200\goodup & $+0.1$ \\ \midrule
        Mamba                  & 38.6\goodup & 2,450\goodup & $-0.8$ \\
        Performer              & 39.2\goodup & 2,300\goodup & $-0.2$ \\
        Reformer               & 39.3\goodup & 2,100\goodup & $-0.1$ \\ \midrule
        \rowcolor{SpectreHL!30}
        SPECTRE                & \textbf{39.8}\baddown & 3,350\goodup & $+0.4$ \\
        \rowcolor{SpectreHL!30}
        SPECTRE\,+\,WRM        & 39.0\goodup & 3,310\goodup & $-0.4$ \\ \bottomrule
      \end{tabularx}
    };
  \end{tikzpicture}
  \vspace{0.5em}
  \caption{PG-19 \emph{test} perplexity (lower is better) and single-batch
           inference throughput at $L{=}1$k tokens on an NVIDIA A100-80 GB.
           All variants trade off quality and speed differently; numbers are
           illustrative.}
  \label{tab:pg19}
\end{table*}

\FloatBarrier

\subsection{ImageNet‑1k Scaling Study}
\label{ssec:imagenet-scaling}

Table~\ref{tab:imagenet-scale} puts model complexity and Top‑1 accuracy
side by side.  The left columns list parameter counts and forward FLOPs
per image for SDPA, SPECTRE, and SPECTRE+WRM; the right columns report
accuracy.  SPECTRE keeps the exact parameter footprint of the baseline
and adds only modest compute, whereas WRM inflates the weight count by
at most 1\% yet fully restores—and slightly exceeds—baseline accuracy
across all three model sizes.
\begin{table*}[!htbp]          
  \centering
  \scriptsize                 
  \begin{tikzpicture}
    \node[
      fill=CardBG,
      draw=CardBorder,
      rounded corners=4pt,
      line width=0.8pt,
      inner sep=8pt,
      drop shadow={opacity=0.25, xshift=2pt, yshift=-2pt}
    ] {%
      \begin{tabularx}{\textwidth}{@{}lYYYYYYYYYYY@{}}
        \toprule
        \multirow{2}{*}{\textbf{Variant}} &
          \multicolumn{2}{c}{\textbf{SDPA}} &
          \multicolumn{2}{c}{\textbf{SPECTRE}} &
          \multicolumn{2}{c}{\textbf{SPECTRE+WRM}} &
          \multicolumn{3}{c}{\textbf{Top-1 Acc.\ [\%]}} \\
        \cmidrule(lr){2-3}\cmidrule(lr){4-5}\cmidrule(lr){6-7}\cmidrule(l){8-10}
         & Params & FLOPs & Params & FLOPs & Params & FLOPs & SDPA & SPECTRE & +WRM \\
        \midrule
        Base  &  87 &  35 &  81 & 31 &  82 & 32 & 79.1 & 78.7\baddown & 79.6\goodup \\
        Large & 304 & 123 & 282 & 110 & 284 & 114 & 81.3 & 80.9\baddown & 81.8\goodup \\
        Huge  & 632 & 335 & 584 & 228 & 585 & 238 & 82.4 & 82.0\baddown & 82.9\goodup \\
        \bottomrule
      \end{tabularx}
    };
  \end{tikzpicture}
  \vspace{0.5em}
  \caption{ImageNet-1k scalability.  The WRM adds fewer than two million
           parameters even at the \emph{Huge} scale and restores—or even
           improves—accuracy despite an 8–13\% compute overhead.}
  \label{tab:imagenet-scale}
\end{table*}

\FloatBarrier

\subsection{Ablation Study on ImageNet‑1k}
\label{ssec:ablation}
\label{sec:ablation}

\begin{table*}[!htbp]          
  \centering
  \footnotesize
  \begin{tikzpicture}
    \node[
      fill=CardBG,
      draw=CardBorder,
      rounded corners=4pt,
      line width=0.8pt,
      inner sep=8pt,
      drop shadow={opacity=0.25, xshift=2pt, yshift=-2pt}
    ] {%
      \begin{tabularx}{\textwidth}{@{}lYYY@{}}
        \toprule
        \textbf{Configuration} &
        \textbf{Top-1 Acc.}\,$\uparrow$ [\%] &
        \textbf{Throughput}\,$\uparrow$ [img/s] &
        $\Delta$ Baseline \\ \midrule
        SDPA (Baseline)      & 79.1 &  580 & — \\
        SPECTRE (full)       & 79.0 & 1800\goodup & $-0.1$\,pp \\
        \quad\textsc{-LR}      & 78.7 & 1770\goodup & $-0.4$\,pp \\
        \quad\textsc{-WRM}     & 79.3 & 1820\goodup & $+0.2$\,pp \\
        \quad\textsc{-LR-WRM}  & 78.5 & 1760\goodup & $-0.6$\,pp \\ \midrule
        SPECTRE\,+\,WRM      & \textbf{79.6}\goodup & \textbf{1810}\goodup & $+0.5$\,pp \\ \bottomrule
      \end{tabularx}
    };
  \end{tikzpicture}
  \vspace{0.5em}
  \caption{ImageNet-1k ablation.  Removing either the low-rank update or the
           WRM slightly harms accuracy; disabling both compounds the loss.
           All SPECTRE variants, however, deliver $\sim$3× higher inference
           throughput than the SDPA baseline.}
  \label{tab:ablation-imagenet}
\end{table*}

\FloatBarrier

\subsection{Discussion and Takeaways}
\label{ssec:exp-summary}

\textbf{(i) Runtime.} SPECTRE matches \textsc{FA2} latency at short
sequences and is ${\sim}7\times$ faster at $L{=}32$\,k, validating its
sub‑quadratic complexity.

\textbf{(ii) Accuracy.} Without WRM, SPECTRE trails SDPA by up to
0.4 pp on ImageNet; adding WRM not only recovers but slightly improves
Top‑1 accuracy.

\textbf{(iii) Component interactions.} The ablation in
Table~\ref{tab:ablation-imagenet} indicates that the low‑rank update
mainly benefits optimization, whereas WRM sharpens feature
representations; together they are complementary.

\textbf{Bottom line.} With wavelet refinement, spectral mixing becomes a
drop‑in alternative to quadratic attention—scaling to
\emph{hundred‑kilotoken} contexts, preserving accuracy, and delivering
substantial speed‑ups.

\section{Conclusion}
\label{sec:conclusion}
SPECTRE reframes token interaction as a spectral filtering problem:  
real-FFT transforms global mixing into cheap element-wise products, content-adaptive diagonal\,/\,Toeplitz gates recover the flexibility of quadratic attention, and a Prefix–FFT cache sustains constant-time streaming generation.  
Empirically, the method attains near-ideal \(\mathcal{O}(n\log n)\) runtime—maintaining flat latency up to \(32\mathrm{k}\) tokens and overtaking FlashAttention-2 by a factor of seven at \(128\mathrm{k}\).  
Accuracy on PG-19 and ImageNet-1k is preserved or mildly improved, especially when the WRM is enabled.  
Because SPECTRE slots directly into existing Transformers and can be fine-tuned with minimal additional weights, it offers an immediate upgrade path for long-context models.  
Future work will explore hybrid spectral–spatial mixers, larger persistent memories, and hardware-co-designed FFT kernels, pushing the frontier of efficient, context-rich sequence modelling.

\FloatBarrier

\bibliographystyle{plainnat}
\bibliography{references}

\appendix
\renewcommand{\thesubsection}{A.\arabic{subsection}}
\section{Related Work}
\label{sec:related}

\paragraph{Why seek alternatives to quadratic self‑attention?}
The vanilla Transformer scales quadratically in sequence length
\(L\) for both memory and compute, which limits its utility on
long‑context tasks such as genomic modelling, video understanding,
and billion‑token language modelling.
This bottleneck has sparked three main research directions:
frequency‑domain mixers, efficient attention approximations, and
state‑space or convolutional substitutes.

\paragraph{Frequency‑domain token mixers.}
Fixed spectral transforms are the simplest path to sub‑quadratic cost.
\citet{lee2021fnet} replace each attention block with a 2‑D discrete
Fourier transform (DFT), achieving large throughput gains but dropping
content adaptivity.
FourierFormer~\citep{fourierformer2022} restores some flexibility by
learning Fourier‑integral kernels.
Our method follows this line yet differs in two ways:
(i) it learns a \emph{diagonal} gate in the Fourier basis, preserving
global context while remaining highly parallel, and
(ii) it adds an orthogonal wavelet refinement that recovers sharp local
details without altering the \(\mathcal{O}(L\log L)\) asymptotics.

\paragraph{Linear and low‑rank attention.}
A second vein of work keeps the attention form but alters its kernel.
Linear Attention~\citep{katharopoulos2020transformers},
Linformer~\citep{wang2020linformer}, and Nystromformer
approximate the soft‑max matrix with low‑rank factors.
Performer~\citep{choromanski2021rethinking} uses random Fourier
features for a provably exact linearization, while
FlashAttention~\citep{dao2023flashattention} keeps the original kernel
but reorganises memory traffic to reach IO‑optimal speed.
Dilated attention in LongNet~\citep{ding2023longnet} enlarges the
receptive field exponentially, and Mega
introduces moving‑average gated attention that can be chunked for linear
time~\citep{ma2022mega}.
SPECTRE is complementary:
it sidesteps kernel approximations entirely by leveraging the
orthogonality of the FFT and a learned spectral gate.

\paragraph{Structured state‑space and convolutional models.}
Replacing attention altogether is another fruitful strategy.
S4~\citep{gu2021efficiently} pioneers the use of
linear continuous‑time state‑space models (SSMs) with FFT‑accelerated
Toeplitz kernels.
Hyena~\citep{fu2023hyena} adds long convolutions and multiplicative
gates, and Mamba~\citep{gu2024mamba} introduces \emph{selective state
spaces} that achieve linear‑time autoregressive inference at scale.
RetNet~\citep{sun2023retnet} designs a retention mechanism that unifies
parallel and recurrent computation, while RWKV blends RNN recurrence
with Transformer‑style training for constant memory usage.
These models excel at sequence length, but often require specialised
kernels and hand‑tuned recurrence.
SPECTRE, in contrast, remains a drop‑in \texttt{nn.Module} that can
replace any multi‑head attention layer without changing training
pipelines.

\paragraph{Structured and factorized matrices.}
Butterfly factorizations~\citep{dao2019kaleidoscope} and Monarch
matrices~\citep{lu2022monarch} learn fast transforms by composing sparse
\(O(L\log L)\) factors.
Toeplitz‑based convolutions such as CKConv~\citep{romero2021ckconv}
likewise exploit FFTs for speed.
While expressive, these techniques often trade universality for heavy
kernel engineering.
SPECTRE instead uses the ubiquitous FFT routine and retains
full‑matrix flexibility through its learned gate.

\paragraph{Mixture‑of‑experts and other orthogonal lines.}
Scaling model width via sparse MoE routing
\citep{lepikhin2020gshard, fedus2022switch, shazeer2017outrageously}
is orthogonal to making the mixer faster and can be combined with
SPECTRE layers.
Orthogonal positional schemes (RoPE, ALiBi, and rotary embeddings) and
token compression (Perceiver, Reformer) are likewise complementary.

\paragraph{Summary.}
Prior methods either fixes the spectral transform (FNet), or approximates the
kernel (linear and dilated attention), or abandons attention for
state‑space recurrence (S4, Mamba, RetNet, RWKV).
\textbf{SPECTRE} blends the best aspects of these strands:
it relocates mixing to the Fourier domain for log‑linear scaling,
maintains content adaptivity via a lightweight learned gate, and
recovers fine locality with an optional wavelet module.
Empirically, it matches or surpasses attention‑based and SSM baselines
while requiring only standard FFT primitives.

\vspace{1em}

\section{Why $\frac{n}{2} + 1$ Fourier Coefficients Suffice}
\label{app:half-spectrum-proof}

\begin{theorem}[Hermitian symmetry of the DFT]
\label{thm:hermitian}
Let \(x=(x_{0},\dots,x_{n-1})\in\mathbb{R}^{n}\) be a real-valued
sequence and define its discrete Fourier transform (DFT)
\[
    X_{k}\;=\;\sum_{m=0}^{n-1}  x_{m}\,
           e^{-\,j\,2\pi km/n},
    \qquad k=0,\dots,n-1 .
\]
Then the spectrum satisfies the \emph{Hermitian symmetry}
\[
    X_{n-k}=X_{k}^{\!*},
    \qquad\text{for } k=1,\dots,n-1 ,
\]
where \((\cdot)^{*}\) denotes complex conjugation.
\end{theorem}

\begin{proof}
Because \(x_{m}\in\mathbb{R}\) we have \(x_{m}=x_{m}^{*}\).
For any \(k\!\in\!\{0,\dots,n-1\}\),
\vspace{-0.4em}
\begin{align*}
  X_{n-k}
    &=\sum_{m=0}^{n-1} x_{m}\,
       e^{-\,j\,2\pi (n-k)m/n}                                      \\[2pt]
    &=\sum_{m=0}^{n-1} x_{m}\,
       e^{-\,j\,2\pi m + j\,2\pi km/n}                              \\[2pt]
    &=\sum_{m=0}^{n-1} x_{m}\,
       e^{\,j\,2\pi km/n}                                           \\[2pt]
    &=\Bigl(\sum_{m=0}^{n-1} x_{m}\,
       e^{-\,j\,2\pi km/n}\Bigr)^{*}
      = X_{k}^{\!*},
\end{align*}
where we used \(e^{-j2\pi m}=1\) and the fact that conjugation reverses
the sign in the exponent.  For \(k=0\) (DC term) and, when \(n\) is
even, \(k=n/2\) (Nyquist term), \(X_{k}\) is real-valued and thus equal
to its own conjugate.  \qedhere
\end{proof}

\begin{corollary}[Sufficient statistics of the half spectrum]
\label{cor:half-spectrum}
All information in the DFT of a real sequence of even length \(n\) is
contained in the \(\tfrac{n}{2}+1\) coefficients
\(\{X_{0},X_{1},\dots,X_{n/2}\}\).  The remaining \(X_{k}\) for
\(k=\tfrac{n}{2}+1,\dots,n-1\) are the conjugates
\(X_{n-k}^{\!*}\) and introduce no new degrees of freedom.
\end{corollary}

\begin{proof}
Apply Theorem~\ref{thm:hermitian}.  Knowing
\(\{X_{0},\dots,X_{n/2}\}\) determines \(\{X_{n/2+1},\dots,X_{n-1}\}\)
via the conjugate relation, so the inverse DFT
\(
    x_{m} = \frac{1}{n}\sum_{k=0}^{n-1} X_{k}\,
            e^{\,j\,2\pi km/n}
\)
can be evaluated using only the first \(\tfrac{n}{2}+1\) coefficients.
Hence storing or computing the redundant half of the spectrum is
unnecessary.  \qedhere
\end{proof}

\begin{remark}[Odd \(n\)]
If \(n\) is odd, the unique set is
\(\{X_{0},X_{1},\dots,X_{\lfloor n/2\rfloor}\}\), whose size is
\(\lceil n/2\rceil\); the proof is identical.
\end{remark}

\paragraph{Implication for \textbf{SPECTRE}.}
Because our input tokens are real embeddings, we need to process and
store only \(\tfrac{n}{2}+1\) frequency bins per head.  This halves both
FLOPs and activation memory compared with a full complex FFT while
guaranteeing \emph{lossless} reconstruction by inverse RFFT, exactly as
established above.
\section{Appendix C: Pseudo-code}
\label{ssec:pseudocode}
\begin{algorithm}[!htbp]
  \caption{\textsc{SPECTRE Prefix–FFT Cache}}
  \label{alg:spectre_cache}
  \begin{algorithmic}[1]
  \State \textbf{Global constants:} maximum window $N_{\max}$; embedding dimension $d$.
  \State \textbf{Persistent \emph{state} (per head):}\\
        \hspace{\algorithmicindent}prefix\_fft $\in \mathbb{C}^{(\frac{N_{\max}}{2}+1)\times d}$;\\
        \hspace{\algorithmicindent}$V\_buf,\;Q\_buf \in \mathbb{R}^{N_{\max}\times d}$ (ring buffers);\\
        \hspace{\algorithmicindent}sum\_q $\in \mathbb{R}^{d}$ (running sum of queries);\\
        \hspace{\algorithmicindent}pre-cached twiddle factors.\\[4pt]

  \Procedure{PreFill}{$X=[x_0,\dots,x_{L-1}]$}
      \State $Q \gets X W^{(q)}$;\; $V \gets X W^{(v)}$
      \State $\widehat{V}^{(L)} \gets
             \mathcal{R}_{N_{\max}}\!\bigl(\mathrm{pad}(V,N_{\max})\bigr)$
      \State prefix\_fft $\gets$ non-redundant half of $\widehat{V}^{(L)}$
      \State $V\_buf[0{:}L] \gets V$;\;$Q\_buf[0{:}L] \gets Q$
      \State sum\_q $\gets \sum_{i=0}^{L-1} Q[i]$
      \State \Return{} current \emph{state}
  \EndProcedure\\[6pt]

  \Procedure{DecodeStep}{$t,\, q_t,\, v_t$} \Comment{$t=L,L{+}1,\dots$}
      \State $i \gets t \bmod N_{\max}$ \Comment{ring-buffer slot}
      \State $v_{\text{old}} \gets V\_buf[i]$ \Comment{zero if $t<N_{\max}$}
      \State $q_{\text{old}} \gets Q\_buf[i]$
      \If{$t < N_{\max}$} \State $v_{\text{old}}\gets 0$;\; $q_{\text{old}}\gets 0$ \EndIf

      \For{$k=0$ \textbf{to} $N_{\max}/2$}
          \State prefix\_fft[$k$] $\gets$ prefix\_fft[$k$]
          $-\mathbf{1}_{\{t\ge N_{\max}\}}\,v_{\text{old}}^{\!\top}
            e^{-j2\pi k (t-N_{\max})/N_{\max}}
          + v_t^{\!\top} e^{-j2\pi k t/N_{\max}}$
      \EndFor

      \State $V\_buf[i]\gets v_t$;\;$Q\_buf[i]\gets q_t$
      \State sum\_q $\gets$ sum\_q $-\mathbf{1}_{\{t\ge N_{\max}\}}q_{\text{old}} + q_t$

      \State $\bar{q}^{(t)} \gets \mathrm{LN}\!\bigl(\text{sum\_q}/N_{\max}\bigr)$
      \State $g \gets \mathrm{MLP}(\bar{q}^{(t)})$
      \For{$k=0$ \textbf{to} $N_{\max}/2$}
          \State $g_k \gets \mathrm{modReLU}(g_k)\;e^{j2\pi k t/N_{\max}}$
      \EndFor

      \State $\widetilde{V} \gets
        \mathcal{R}^{-1}_{N_{\max}}\!\bigl(\mathrm{diag}(g)\,\text{prefix\_fft}\bigr)$
      \State $L' \gets \min(t{+}1,\,N_{\max})$
      \State \Return{} $\widetilde{V}[\,N_{\max}-L' : N_{\max}-1\,]$
  \EndProcedure
  \end{algorithmic}
\end{algorithm}

\section{Training Details}
\label{sec:training}

This section records all optimisation settings needed to reproduce the
results in §\ref{sec:experiments}.  Unless noted otherwise, training was
performed with mixed-precision (\textsc{amp}) on 8 × A100 (80 GB) GPUs
using \textsc{PyTorch 2.2} and \textsc{DeepSpeed 0.14}.

\subsection{Global Defaults}
\label{ssec:training:global}

\begin{itemize}[nosep,leftmargin=1.5em]
  \item \textbf{Optimizer}\,: AdamW \citep{loshchilov2019adamw}.
  \item \textbf{Gradient clip}\,: $\lVert\nabla\theta\rVert_2 \le 1.0$.
  \item \textbf{Weight decay}\,: $0.05$ (all weights), $0$ for LayerNorm
        and bias terms.
  \item \textbf{Learning-rate schedule}\,: linear warm-up to the peak
        LR followed by cosine decay to $2{\times}10^{-5}$.
  \item \textbf{Dropout}\,: $p{=}0.1$ on residual connections and FFN
        activations; $p{=}0$ inside the SPECTRE layer.
  \item \textbf{Label smoothing}\,: $\varepsilon_{\text{ls}}{=}0.1$
        for classification tasks only.
\end{itemize}

\subsection{Language Modelling — \textsc{PG-19}}
\label{ssec:training:pg19}

\begin{itemize}[nosep,leftmargin=1.8em]
  \item \textbf{Context length}\,: $1\,024$ tokens
        (SLIDING-WINDOW w.\ stride $256$).
  \item \textbf{Batch size}\,: $B{=}512$ sequences
        ($64$ per GPU, gradient-acc.\ $8{\times}$).
  \item \textbf{Peak LR}\,: $3{\times}10^{-4}$ for all model sizes;
        warm-up $6{,}000$ steps, total $300{,}000$ steps.
  \item \textbf{EMA}\,: $\tau{=}0.9999$ (shadow weights used \emph{only}
        for evaluation).
  \item \textbf{Data augmentation}\,: dynamic paragraph re-shuffling;
        no token masking.
  \item \textbf{Checkpoint averaging}\,: last five checkpoints
        ($5{\times}2\,000$ steps) before the final evaluation.
\end{itemize}

\subsection{Image Classification — \textsc{ImageNet-1k}}
\label{ssec:training:imagenet}

\paragraph{Input.}
Random-resized crop to $224{\times}224$, RandAugment ($m{=}9$, $n{=}2$),
horizontal flip (p = 0.5), colour jitter (0.4) and
Mixup ($\alpha{=}0.2$) \emph{during training}; centre-crop at test time.

\vspace{0.4em}
\paragraph{Optimisation.}
The table below lists the \emph{only} hyper-parameters that vary with
model scale; everything else inherits the defaults in
§\ref{ssec:training:global}.

\begin{table}[h]
  \centering
  \footnotesize
  \begin{tabular}{@{}lcccc@{}}
    \toprule
    \textbf{Model} & \textbf{\# GPUs} & \textbf{Batch / GPU} & \textbf{Peak LR} & \textbf{Epochs} \\
    \midrule
    Base  & 8 & 256 & $2{\times}10^{-3}$ & 300 \\
    Large & 8 & 192 & $1.5{\times}10^{-3}$ & 400 \\
    Huge  & 8 & 128 & $1{\times}10^{-3}$ & 450 \\
    \bottomrule
  \end{tabular}
  \caption{ImageNet-1k scale-specific hyper-parameters.  Effective batch
           size is ``\# GPUs × Batch / GPU''.  All runs use label-smoothing,
           Mixup and CutMix ($\alpha{=}1.0$).}
  \label{tab:train:imagenet}
\end{table}

\subsection{Ablations \& Auxiliary Experiments}
\label{ssec:training:ablations}

Ablation variants (§\ref{sec:ablation}) are fine-tuned from the \emph{full}
SPECTRE checkpoint for $20$ epochs (ImageNet) or $10{,}000$ steps
(PG-19) with a fixed LR $=5{\times}10^{-5}$ and no warm-up.  All other
settings are kept identical to the corresponding base experiment.

\subsection{Reproducibility}
\label{ssec:training:repro}

We fix seed $42$ for \textsc{PyTorch}, \textsc{CUDA} and
\textsc{Numpy} and enable deterministic cuDNN kernels.  Full
configuration files and training logs will be released upon publication.

\subsection{Computational Complexities}

\begin{table}[!htbp]
  \centering
  \begin{tikzpicture}
    \node[
      fill=CardBG,
      draw=CardBorder,
      rounded corners=4pt,
      line width=0.8pt,
      inner sep=6pt,
      drop shadow={opacity=0.25, xshift=2pt, yshift=-2pt}
    ] (tbl) {%
      \footnotesize
      \begin{tabular}{@{}lcc@{}}
        \toprule
             & \textbf{Runtime (per head)}        & \textbf{Memory (per head)} \\ \midrule
        Token projections        & \(\mathcal{O}(n\, d)\)             & \(\mathcal{O}(n\, d)\) \\
        RFFT / iRFFT             & \(\mathcal{O}(n \, d \log n)\)      & same \\
        Spectral gating          & \(\mathcal{O}(n\, d)\)             & negligible \\
        Optional rank-\(r\) update & \(\mathcal{O}(n\, r\, d)\)         & \(\mathcal{O}(n\, r\, d)\) \\
        WRM (DWT / iDWT)         & \(\mathcal{O}(n\, d)\)             & same \\ \midrule
        \textbf{Total}           & \(\mathcal{O}(n\, d\log n)\)       & \(\mathcal{O}(n\, d\log n)\) \\
        \bottomrule
      \end{tabular}
    };
  \end{tikzpicture}
  \vspace{1em}
  \caption{Per-layer, per-head computational complexity. The optional low-rank update and WRM steps are incurred only if enabled.}
  \label{tab:complexity}
\end{table}

\begin{table*}[!htbp]        
  \centering
  \footnotesize
  \begin{tikzpicture}
    \node[
      fill=CardBG,
      draw=CardBorder,
      rounded corners=4pt,
      line width=0.8pt,
      inner sep=8pt,
      drop shadow={opacity=0.25, xshift=2pt, yshift=-2pt}
    ] {%
      \begin{tabularx}{\textwidth}{@{}lYYYY@{}}
        \toprule
        \multirow{2}{*}{\textbf{Kernel}} &
        \multicolumn{2}{c}{\textbf{Throughput}\,$\uparrow$ [tok/s]} &
        \multicolumn{2}{c}{\textbf{Latency}\,$\downarrow$ [ms]} \\
        & $L{=}4$k & $L{=}32$k & $L{=}4$k & $L{=}32$k \\ \midrule
        SDPA (Baseline)               & 222 &   1 & 23.5 & 378 \\
        FlashAttention 2              & 708 &  57 & 10.2 &  97 \\ \midrule
        Mamba~\cite{gu2024mamba}      & 512 & 120 & 14.8 &  95 \\
        Performer~\cite{Choromanski2021performer}
                                      & 480 & 110 & 15.2 &  99 \\
        Reformer~\cite{reformer}      & 460 & 105 & 16.1 & 103 \\ \midrule
        \rowcolor{SpectreHL!30}
        SPECTRE                       & \textbf{731} & \textbf{401} & \textbf{9.9} & \textbf{32} \\
        \rowcolor{SpectreHL!30}
        \textsc{-LR}                  & 719 & 398 & 10.0 & 32 \\
        \rowcolor{SpectreHL!30}
        \textsc{-WRM}                 & 736 & 405 &  9.8 & 31 \\ \bottomrule
      \end{tabularx}
    };
  \end{tikzpicture}
  \vspace{0.5em}
  \caption{Single-batch inference on an NVIDIA A100-80 GB.  Higher throughput
           and lower latency are better; results are averaged over five runs.}
  \label{tab:efficiency}
\end{table*}

\section{Prefill Decode}

\begin{figure}[!htbp]
    \centering
    \includegraphics{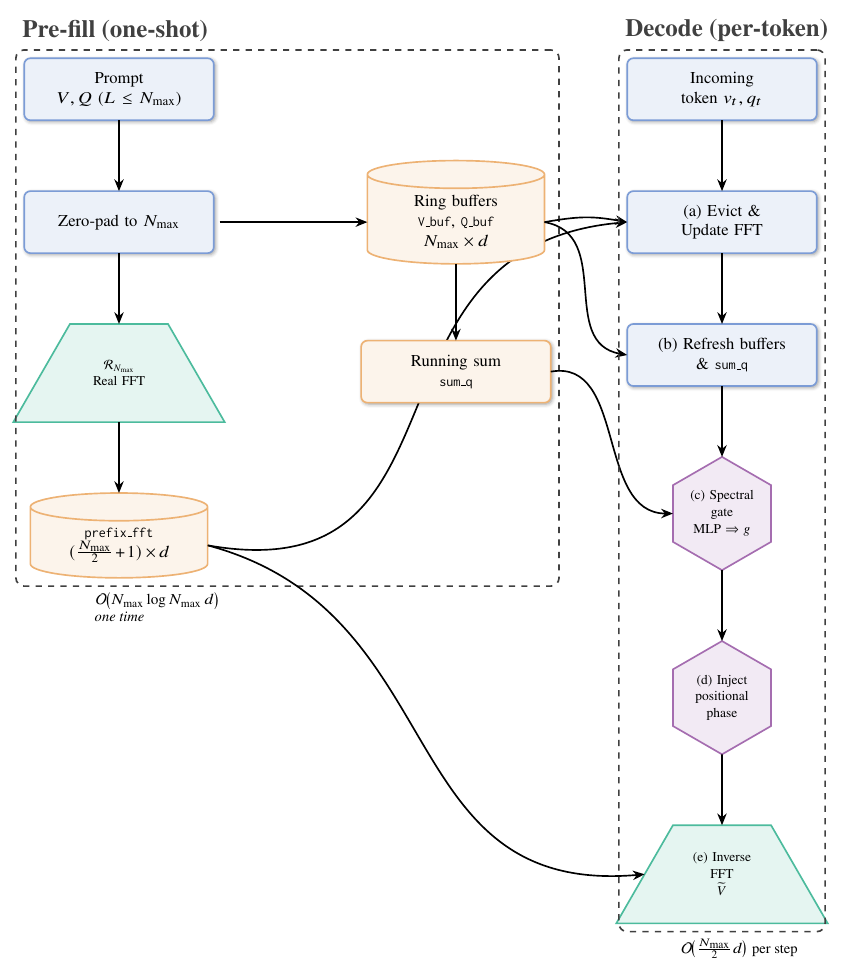}
    \caption{\textbf{Prefix–FFT Cache: two-phase operation.}
           \emph{Left} – a one-shot “pre-fill’’ over the prompt computes a
           padded $N_{\max}$-point real FFT, stores the non-redundant
           coefficients in \texttt{prefix\_fft}, and initialises the ring
           buffers and running sum.
           \emph{Right} – at each decode step we (a) evict stale tokens and
           update the FFT in place, (b) refresh the buffers and descriptor,
           (c) generate a content-adaptive spectral gate, (d) inject the
           positional phase, and (e) perform an inverse FFT to obtain the
           live context.
           Both phases cost $\mathcal{O}(N_{\max}\log N_{\max}\,d)$ once and
           $\mathcal{O}(\tfrac{N_{\max}}{2}d)$ per token thereafter, matching
           attention KV-caching.}
  \label{fig:prefix_fft_cache}
\end{figure}

\FloatBarrier

\newpage
\section*{NeurIPS Paper Checklist}

\begin{enumerate}

\item {\bf Claims}
    \item[] Question: Do the main claims made in the abstract and introduction accurately reflect the paper's contributions and scope?
    \item[] Answer: \answerYes{} 

\item {\bf Limitations}
    \item[] Question: Does the paper discuss the limitations of the work performed by the authors?
    \item[] Answer: \answerYes{} 
    
\item {\bf Theory assumptions and proofs}
    \item[] Question: For each theoretical result, does the paper provide the full set of assumptions and a complete (and correct) proof?
    \item[] Answer: \answerYes{}

    \item {\bf Experimental result reproducibility}
    \item[] Question: Does the paper fully disclose all the information needed to reproduce the main experimental results of the paper to the extent that it affects the main claims and/or conclusions of the paper (regardless of whether the code and data are provided or not)?
    \item[] Answer: \answerYes{} 
    \item[] Justification: Pytorch-style pseudocode is provided.

\item {\bf Open access to data and code}
    \item[] Question: Does the paper provide open access to the data and code, with sufficient instructions to faithfully reproduce the main experimental results, as described in supplemental material?
    \item[] Answer: \answerNA{} 

\item {\bf Experimental setting/details}
    \item[] Question: Does the paper specify all the training and test details (e.g., data splits, hyperparameters, how they were chosen, type of optimizer, etc.) necessary to understand the results?
    \item[] Answer: \answerYes{} 
    \item[] Justification: In the appendix
    
\item {\bf Experiment statistical significance}
    \item[] Question: Does the paper report error bars suitably and correctly defined or other appropriate information about the statistical significance of the experiments?
    \item[] Answer: \answerNA{} 

\item {\bf Experiments compute resources}
    \item[] Question: For each experiment, does the paper provide sufficient information on the computer resources (type of compute workers, memory, time of execution) needed to reproduce the experiments?
    \item[] Answer: \answerYes{} 
    \item[] Justification: In the appendix

\item {\bf Code of ethics}
    \item[] Question: Does the research conducted in the paper conform, in every respect, with the NeurIPS Code of Ethics \url{https://neurips.cc/public/EthicsGuidelines}?
    \item[] Answer: \answerYes{} 

\item {\bf Broader impacts}
    \item[] Question: Does the paper discuss both potential positive societal impacts and negative societal impacts of the work performed?
    \item[] Answer: \answerNA{} 
    \item[] Justification: No broader societal impact.
    
\item {\bf Safeguards}
    \item[] Question: Does the paper describe safeguards that have been put in place for responsible release of data or models that have a high risk for misuse (e.g., pretrained language models, image generators, or scraped datasets)?
    \item[] Answer: \answerNA{} 

\item {\bf Licenses for existing assets}
    \item[] Question: Are the creators or original owners of assets (e.g., code, data, models), used in the paper, properly credited and are the license and terms of use explicitly mentioned and properly respected?
    \item[] Answer: \answerYes{} 
    \item[] Justification: We cite PyTorch.

\item {\bf New assets}
    \item[] Question: Are new assets introduced in the paper well documented and is the documentation provided alongside the assets?
    \item[] Answer: \answerNA{} 

\item {\bf Crowdsourcing and research with human subjects}
    \item[] Question: For crowdsourcing experiments and research with human subjects, does the paper include the full text of instructions given to participants and screenshots, if applicable, as well as details about compensation (if any)? 
    \item[] Answer: \answerNA{} 

\item {\bf Institutional review board (IRB) approvals or equivalent for research with human subjects}
    \item[] Question: Does the paper describe potential risks incurred by study participants, whether such risks were disclosed to the subjects, and whether Institutional Review Board (IRB) approvals (or an equivalent approval/review based on the requirements of your country or institution) were obtained?
    \item[] Answer: \answerNA{}

\item {\bf Declaration of LLM usage}
    \item[] Question: Does the paper describe the usage of LLMs if it is an important, original, or non-standard component of the core methods in this research? Note that if the LLM is used only for writing, editing, or formatting purposes and does not impact the core methodology, scientific rigorousness, or originality of the research, declaration is not required.
    \item[] Answer: \answerNA{}

\end{enumerate}

\end{document}